\documentclass[fleqn]{llncs}
 \pdfoutput=1 
\usepackage{url}
\usepackage{mathtools}
\usepackage{times}
\usepackage{amssymb} 
\usepackage{helvet}
\usepackage{courier}
\usepackage{graphicx}
\usepackage{graphics}
\usepackage{multirow}
\usepackage{amsmath}
\usepackage{multirow}
\usepackage{color}

\newcommand{\comment}[1]{}

\newcommand{\etal}{{\it et al.}}

\newcommand{\SB}{\textsf{sb}}

\newcommand{\bee}{\textsf{BEE}}

\newcommand{\sset}[2]{\left\{~#1  \left|
      \begin{array}{l}#2\end{array}
    \right.     \right\}}

\renewcommand{\AA}{{\cal A}}

\newcommand{\RR}{{\cal R}}
\newcommand{\MM}{{\cal M}}

\title{Solving Graph Coloring Problems with \\Abstraction and
  Symmetry\thanks{Supported by the Israel Science Foundation, grant
    182/13. Computational resources provided by an IBM Shared
    University Award (Israel).}}

\author{Michael Codish\inst{1} \and Michael Frank\inst{1} \and 
        Avraham Itzhakov\inst{1}\and  Alice Miller\inst{2}}
\institute{
  Department of Computer Science,
  Ben-Gurion University of the Negev, Israel
\and
  School of Computing Science,
  University of Glasgow, Scotland
}

\begin{document}
\maketitle
\begin{abstract}
  This paper introduces a general methodology, based on abstraction
  and symmetry, that applies to solve hard graph edge-coloring
  problems and demonstrates its use to provide further evidence that
  the Ramsey number $R(4,3,3)=30$.
  The number $R(4,3,3)$ is often presented as the unknown Ramsey
  number with the best chances of being found ``soon''. Yet, its
  precise value has remained unknown for more than 50 years.
  We illustrate our approach by showing that:
  (1) there are precisely 78{,}892 $(3,3,3;13)$ Ramsey colorings; and
  (2) if there exists a $(4,3,3;30)$ Ramsey coloring then it is
  (13,8,8) regular. Specifically each node has 13 edges in the first color,
  8 in the second, and 8 in the third. 
  We conjecture that these two results will help provide a proof that
  no $(4,3,3;30)$ Ramsey coloring exists implying that $R(4,3,3)=30$.
\end{abstract}


\section{Introduction}\label{sec:intro}

This paper introduces a general methodology that applies to solve
graph edge-coloring problems and demonstrates its application in the
search for Ramsey numbers. These are notoriously hard graph coloring
problems that involve assigning $k$ colors to the edges of a complete
graph. In particular, $R(4,3,3)$ is the smallest number $n$ such that
any coloring of the edges of the complete graph $K_n$ in three colors
will either contain a $K_4$ sub-graph in the first color, a $K_3$
sub-graph in the second color, or a $K_3$ sub-graph in the third
color. The precise value of this number has been sought for more than
50 years.  Kalbfleisch~\cite{kalb66} proved in 1966 that $R(4,3,3)\geq
30$, Piwakowski~\cite{Piwakowski97} proved in 1997 that $R(4,3,3)\leq
32$, and one year later Piwakowski and Radziszowski~\cite{PR98} proved
that $R(4,3,3)\leq 31$. We demonstrate how our methodology applies to
provide further evidence that $R(4,3,3)=30$.

Solving hard search problems on graphs, and graph coloring problems in
particular, relies heavily on breaking symmetries in the search space.
When searching for a graph, the names of the vertices do not
matter, and restricting the search modulo graph isomorphism is highly
beneficial. When searching for a graph coloring, on top of graph
isomorphism, solutions are typically closed under permutations of the
colors: the names of the colors do not matter and the term
often used is ``weak isomorphism''~\cite{PR98} (the equivalence
relation is weaker because both node names and edge colors do not
matter).
When the problem is to compute the set of all solutions modulo (weak)
isomorphism the task is even more challenging. Often one first
attempts to compute all of solutions of the coloring problem, and to
then apply one of the available graph isomorphism tools, such as
\texttt{nauty}~\cite{nauty} to select representatives of their
equivalence classes modulo (weak) isomorphism. However, typically the
number of solutions is so large that this approach is doomed to fail
even though the number of equivalence classes itself is much
smaller. The problem is that tools such as \texttt{nauty} apply after,
and not during, search.
To this end, we first observe that the technique described in
\cite{DBLP:conf/ijcai/CodishMPS13} for graph isomorphism applies also
to weak isomorphism, facilitating symmetry breaks during the search
for solutions to graph coloring problems. This form
of symmetry breaking is an important component in our methodology but
on its own cannot provide solutions to hard graph coloring problems.

When confronted with hard computational problems, a common strategy is
to consider approximations which focus on ``abstract'' solutions which
characterize properties of the actual ``concrete'' solutions. To this
end, given a graph coloring problem with $k$ colors on $n$ nodes, we
introduce the notion of an $n\times k$ \emph{degree matrix} in which
each of $n$ rows describes the degrees of a coresponding node in the
$k$ colors. In case the graph coloring problem is too hard to solve
directly, we seek, possibly an over approximation of, all of the
degree matrices of its solutions. This enables a subsequent
independent search of solutions ``per degree matrix'' facilitating so
called ``embarrassingly parallel'' search.

After laying the ground for a methodology based on symmetry breaking
and abstraction we apply it to the problem of computing the Ramsey
number $R(4,3,3)$ which reduces to determining if there exists a
$(4,3,3)$ coloring of the complete graph $K_{30}$.
%
We first characterize the degrees of the nodes in each of the three
colors in any such coloring, if one exists. To this end, we show that
if there is such a graph coloring then, up to swapping the colors
two and three, all of its vertices have degrees in the three colors
corresponding to the following triples:
$(13, 8, 8)$, $(14, 8, 7)$, $(15, 7, 7)$, 
$(15, 8, 6)$, $(16, 7, 6)$, $(16, 8, 5)$.
%
Then, we demonstrate that any potential $(4,3,3;30)$ coloring
with a node with degrees $(d_1,d_2,d_3)$ in the corresponding colors must have three corresponding embedded graphs $G_1, G_2, G_3$ which
are $(3,3,3;d_1)$, $(4,2,3;d_2)$, and $(4,3,2;d_3)$ colorings.
For all of the cases except when the degrees are $(13,8,8)$ these sets
of colorings are known and easy to compute.
Based on this, we show using a SAT solver that there can be no nodes
with degrees $(14, 8, 7)$, $(15, 7, 7)$, $(15, 8, 6)$, $(16, 7, 6)$ or
$(16, 8, 5)$ in any $(4,3,3;30)$ coloring. Thus, we prove that any such
coloring would have to be $(13,8,8)$ regular, meaning that all nodes are
of degree 13 in the first color and of degree 8 in the second and
third color.

In order to apply the same proof technique for the case where the
graph is $(13,8,8)$ regular we need to first compute the set of all
$(3,3,3;13)$ colorings, modulo weak isomorphism. This set of graphs
does not appear in previously published work.  
So, we address the problem of computing the set of all $(3,3,3;13)$
Ramsey colorings, modulo weak isomorphism.  This results in a set of
78{,}892 graphs. The set of $(3,3,3;13)$ Ramsey colorings has recently
been independently computed by at least three other researchers:
Richard Kramer, Ivan Livinsky, and Stanislaw
Radziszowski~\cite{stas:personalcommunication}.

Finally, we describe the ongoing computational effort to prove that
there is no $(13,8,8)$ regular $(4,3,3)$ Ramsey coloring. 
Using the embedding approach, and given the 78{,}892 (3,3,3;13)
colorings there are 710{,}028 instances to consider. Over the period
of 3 months we have verified using a SAT solver that 315{,}000 of
these are not satisfiable.
When this effort completes we will know if the value of $R(4,3,3)$ is
30 or 31.

Throughout the paper we express graph coloring problems in terms of
constraints via a ``mathematical language''. Our implementation uses
the \bee, finite-domain constraint compiler~\cite{jair2013}, which
solves constraints by encoding them to CNF and applying an underlying
SAT solver. The solver can be applied to find a single (first)
solution to a constraint, or to find all solutions for a constraint
modulo a specified set of (integer and/or Boolean) variables.
Of course, correctness of our results assumes a lack of bugs in the
tools we have used including the constraint solver and the underlying
SAT solver. To this end
we have performed our computations using four different underlying SAT
solvers: MiniSAT~\cite{minisat,EenS03}, CryptoMiniSAT~\cite{Crypto},
Glucose~\cite{Glucose,AudemardS09}, and
Lingeling~\cite{Lingeling,Biere14}.
\bee\ configures directly with MiniSAT, CryptoMiniSAT, and
Glucose. For the experiments with Lingeling we first apply \bee\ to
generate a CNF (dimacs) file and subsequently invoke the SAT solver.
Lingeling, together with Druplig~\cite{BiereLing}, provides a
proof certificate for unsat instances (and we have taken advantage of
this option).
All computations were performed on a cluster with a total of $228$ Intel
E8400 cores clocked at 2 GHz each, able to run a total of $456$
parallel threads. Each of the cores in the cluster has computational
power comparable to a core on a standard desktop computer.  Each SAT
instance is run on a single thread.

The notion of a ``degree matrix'' arises in the literature
with several different meanings. Degree matrices with the same meaning
as we use in in this paper are considered in
\cite{CarrollIsaak2009}.
Gent and Smith \cite{GentS00}, 
building on the work of Puget \cite{Puget93}, study symmetries in
graph coloring problems and recognize the importance of breaking
symmetries during search.
Meseguer and Torras \cite{MeseguerT01} present a framework 
for exploiting symmetries to heuristically guide a depth first 
search, and show promising results for $(3,3,3;n)$ Ramsey colorings 
with $14\leq n\leq 17$.
Al-Jaam \cite{Jaam07} proposes a hybrid meta-heuristic algorithm
for Ramsey coloring problems, combining tabu search and simulated
annealing. 
While all of these approaches report promising results, to the best
of our knowledge, none of them have been successfully applied to solve
open instances or improve the known bounds on classical Ramsey
numbers.
Our approach focuses on symmetries due to weak-isomorphism for graph
coloring and models symmetry breaking in terms of constraints
introduced as part of the problem formulation.  This idea, advocated by
Crawford \etal~\cite{crawford96}, has previously been explored in
\cite{DBLP:conf/ijcai/CodishMPS13} (for graph isomorphism), and in
\cite{Puget93} (for graph coloring).

Graph coloring has many applications in computer science and
mathematics, such as scheduling, register allocation and
synchronization, path coloring and sensor networks. Specifically, many
finite domain CSP problems have a natural representation as graph
coloring problems.
Our main contribution is a general methodology that applies to solve
graph edge coloring problems. The application to potentially compute
an unknown Ramsey number is attractive, but the importance here is in
that it shows the utility of the methodology.

\section{Preliminaries}


An $(r_1,\ldots,r_k;n)$ Ramsey coloring is an assignment of one of $k$
colors to each edge in the complete graph $K_n$ such that it does not
contain a monochromatic complete sub-graph $K_{r_i}$ in color $i$ for
$1\leq i\leq k$. The set of all such colorings is denoted
$\RR(r_1,\ldots,r_k;n)$.  The Ramsey number $R(r_1,\ldots,r_k)$ is the
least $n>0$ such that no $(r_1,\ldots,r_k;n)$ coloring exists.
In the multicolor case ($k>2$), the only known value of a nontrivial
Ramsey number is $R(3,3,3)=17$.  The value of $R(4,3,3)$ is known to
be equal either to 30 or to 31.  The numbers of $(3,3,3;n)$ colorings
are known for $14\leq n\leq 16$ but prior to this paper the number of
colorings for $n=13$ was unpublished. Recently, the set of all
$(3,3,3;13)$ colorings has also been computed by other researchers
\cite{stas:personalcommunication}, and they number 78{,}892 as
reported also in this paper.
More information on recent results concerning Ramsey numbers can be
found in the electronic dynamic survey by Radziszowski~\cite{Rad}.

In this paper, graphs are always simple, i.e.  undirected and with no
self loops. Colors are associated with graph edges. The set of
neighbors of a node $x$ is denoted $N(x)$ and the set of neighbors by
edges colored $c$, by $N_c(x)$.
For a natural number $n$ denote $[n]=\{1,2,\ldots,n\}$.
A graph coloring, in $k$ colors, is a pair $(G,\kappa)$ consisting of
a simple graph $G=([n],E)$ and a mapping $\kappa\colon E\to[k]$. When
$\kappa$ is clear from the context we refer to $G$ as the graph
coloring.
The sub-graph of $G$ induced by the color $c\in[k]$ is the graph 
$G^c=([n],\sset{e\in E}{\kappa(e)=c})$.
The sub-graph of $G$ on the $c$ colored neighbors of a node $x$ is the
projection of the labeled edges in $G$ to $N_c(x)\times N_c(x)$ and
denoted $G^c_x$. 
We typically represent $G$ as an $n\times n$ adjacency matrix,
$A$, defined such that
\[A_{i,j}=    \begin{cases}\kappa(i,j) & \mbox{if } (i,j) \in E\\
                                     0          & \mbox{otherwise}
                       \end{cases}
\]
If $A$ is the adjacency matrix representing the graph $G$, then we
denote the Boolean adjacency matrix corresponding to $G^c$ as $A[c]$.
We denote the $i^{th}$ row of a matrix $A$ by $A_i$.
The color-$c$ degree of a node $x$ in $G$ is denoted
$deg_{G^c}(x)$ and is equal to the degree of $x$ in the induced
sub-graph~$G^c$. When clear from the context we write $deg_{c}(x)$.
Let $G=([n],E)$ and $\pi$ be a permutation on $[n]$. Then $\pi(G) =
(V,\sset{ (\pi(x),\pi(y))}{ (x,y) \in E})$. Permutations act on
adjacency matrices in the natural way: If $A$ is the adjacency matrix
of a graph $G$, then $\pi(A)$ is the adjacency matrix of $\pi(G)$
obtained by simultaneously permuting with $\pi$ both rows and columns
of $A$.

\begin{figure}
  \centering
    \begin{eqnarray}
      \varphi_{adj}^{n,k}(A) &=& \hspace{-2mm}\bigwedge_{1\leq q<r\leq n}
                          \left(\begin{array}{l}
                             1\leq A_{q,r}\leq k  ~~\land~~
                             A_{q,r} = A_{r,q} ~~\land ~~
                             A_{q,q} = 0
                           \end{array}\right)
     \label{constraint:simple}
\\
     \varphi_{K_3}^{n,c}(A) &=& \hspace{-2mm}\bigwedge_{1\leq q<r<s\leq n}\hspace{-3mm}
                      \neg~ \bigg(A_{q,r} = A_{q,s} =  A_{r,s} = c\bigg)
     \label{constraint:nok3}
\\
     \varphi_{K_4}^{n,c}(A) &=& \hspace{-2mm}\bigwedge_{1\leq q<r<s<t\leq n}\hspace{-4mm}
                       \neg\left(\begin{array}{l}
                          A_{q,r} = A_{q,s} = A_{q,t} = 
                          A_{r,s} = A_{r,t} = A_{s,t} = c
                        \end{array}\right)
     \label{constraint:nok4}
\end{eqnarray}

\begin{eqnarray}
\label{constraint:r333}
    \varphi_{(3,3,3;n)}(A) & = & \varphi_{adj}^{n,3}(A) \land  \hspace{-2mm}
                      \bigwedge_{1\leq c\leq 3} \hspace{-1mm}
                      \varphi_{K_3}^{n,c}(A) \\
\label{constraint:r334}
    \varphi_{(4,3,3;n)}(A) & = &  \varphi_{adj}^{n,3}(A) \land  \hspace{-2mm}
                      \bigwedge_{1\leq c\leq 2} \hspace{-1mm}
                      \varphi_{K_3}^{n,c}(A) \land\varphi_{K_4}^{n,3}(A)~~
\end{eqnarray}

    \caption{Graph labeling problems: Ramsey colorings $(3,3,3;n)$ and
      $(4,3,3;n)$}
  \label{fig:gcp}
\end{figure}

A graph coloring problem is a formula $\varphi(A)$ where $A$ is an
$n\times n$ adjacency matrix of integer variables together with a
set (conjunction) of constraints $\varphi$ on these variables. A
solution is an assignment of integer values to the variables in $A$
which satisfies $\varphi$ and determine both the graph edges and their
colors. We often refer to a solution as an integer adjacency matrix
and denote the set of solutions as $sol(\varphi(A))$.
Figure~\ref{fig:gcp} illustrates the two graph coloring problems we
focus on in this paper: $(3,3,3;n)$ and $(4,3,3;n)$ Ramsey colorings.
In Constraint~(\ref{constraint:simple}), $\varphi_{adj}^{n,k}(A)$,
states that the graph $A$ has $n$ vertices, is $k$ colored, and is
simple (symmetric, and with no self loops).
In Constraints~(\ref{constraint:nok3}) and~(\ref{constraint:nok4}),
$\varphi_{K_3}^{n,c}(A)$ and $\varphi_{K_4}^{n,c}(A)$ state that the
$n$ vertex graph $A$ has no embedded sub-graph $K_3$, and respectively
$K_4$, in color $c$.
In Constraints~(\ref{constraint:r333}) and~(\ref{constraint:r334}),
the formulas state that a graph $A$ is a $(3,3,3;n)$ and respectively
a $(4,3,3;n)$ Ramsey coloring.

For graph coloring problems, solutions are typically closed under
permutations of vertices and of colors. Restricting the search space
for a solution modulo such permutations is crucial when trying to
solve hard graph coloring problems. It is standard practice to
formalize this in terms of graph (coloring) isomorphism.


\begin{definition}[\textbf{(weak) isomorphism of graph colorings}]
 \label{def:weak_iso}
 Let $(G,{\kappa_1})$ and $(H,{\kappa_2})$ be $k$-color graph
 colorings with $G=([n],E_1)$ and $H=([n],E_2)$. 
 We say that $(G,{\kappa_1})$ and $(H,{\kappa_2})$ are weakly
 isomorphic, denoted $(G,{\kappa_1})\approx(H,{\kappa_2})$ if 
 there exist permutations $\pi \colon [n] \to [n]$ and 
 $\sigma \colon [k] \to [k]$ such that $(u,v) \in E_1 \iff 
 (\pi(u),\pi(v)) \in E_2$ and $\kappa_1(u,v) = \sigma(\kappa_2(\pi(u),
 \pi(v)))$.
 When $\sigma$ is the identity permutation, (i.e. $\kappa_1(u,v) =
 \kappa_2(\pi(u),\pi(v))$) we say that $(G,{\kappa_1})$ and
 $(H,{\kappa_2})$ are isomorphic.
 We denote such a weak isomorphism thus: $(G,{\kappa_1})\approx_{\pi,\sigma}(H,{\kappa_2})$.
\end{definition}
 
The following lemma emphasizes the importance of weak graph
isomorphism as it relates to Ramsey numbers. Many classic
coloring problems exhibit the same property.

\begin{lemma}[\textbf{$\RR(r_1,r_2,\ldots,r_k;n)$ is closed under $\approx$}]
  \quad Let $(G,{\kappa_1})$ and $(H,{\kappa_2})$ be graph colorings
  in $k$ colors such that $(G,\kappa_1) \approx_{\pi,\sigma}
  (H,\kappa_2)$. Then, \[(G,\kappa_1) \in \RR(r_1,r_2,\ldots,r_k;n)$
  $\iff$ $(H,\kappa_2) \in
  \RR(\sigma(r_1),\sigma(r_2),\ldots,\sigma(r_k);n).\]
\end{lemma}
\begin{proof}
  Assume that $(G,\kappa_1) \in \RR(r_1,r_2,\ldots,r_k;n)$ and
  in contradiction that $(H,\kappa_2) \notin
  \RR(\sigma(r_1),\sigma(r_2),\ldots,\sigma(r_k);n)$. Let $R$ denote a
  monochromatic clique of size $r_s$ in $H$ and $R^{-1}$ the inverse
  of $R$ in $G$.
  From Definition~\ref{def:weak_iso},
  $(u,v) \in R \iff (\pi^{-1}(u), \pi^{-1}(v))\in R^{-1}$ and
  $\kappa_2(u,v) = \sigma^{-1}(\kappa_1(u,v))$.  Consequently $R^{-1}$ is a
  monochromatic clique of size $r_s$ in $(G,\kappa_1)$ in contradiction
  to $(G,\kappa_1)$ $\in$ $\RR(r_1,r_2,\ldots,r_k;n)$.
\end{proof}

Codish \etal~introduce in \cite{DBLP:conf/ijcai/CodishMPS13} an
approach to break symmetries due to graph isomorphism (without colors)
during the search for a solution to general graph problems. Their
approach involves adding a symmetry breaking predicate
$\SB^*_\ell(A)$, as advocated  by Crawford
\etal~\cite{crawford96}, on the variables of the adjacency matrix,
$A$, when solving graph problems. In \cite{CodishMPS14} the authors
show that the symmetry breaking approach of
\cite{DBLP:conf/ijcai/CodishMPS13} holds also for graph coloring
problems where the adjacency matrix consists of integer variables (the
proofs for the integer case are similar to those for the Boolean case).

\begin{definition}\textbf{\cite{DBLP:conf/ijcai/CodishMPS13}.}
\label{def:SBlexStar}
Let $A$ be an $n\times n$ adjacency matrix. Then, viewing the rows of
$A$ as strings,
  \[\SB^*_\ell(A) = \bigwedge_{i<j}  
  A_{i}\preceq_{\{i,j\}}A_{j}\] where $s\preceq_{\{i,j\}}s'$ is the
  lexicographic order on strings $s$ and $s'$ after simultaneously omitting
  the elements at positions $i$ and $j$.
\end{definition}
\begin{table}[t]
  \centering{\scriptsize
  \begin{tabular}{|r|r|rrr|rrrr|}
    \hline
    $n$ & \#${\setminus}_{\approx}$ & \multicolumn{3}{c|}{no sym break} 
                                   & \multicolumn{4}{c|}{with sym break}
    \\
    \hline
      &&      \#vars & \#clauses & time          & \#vars & \#clauses & time & \#\\
    \hline
    17 & 0    &408 & 2584 &~~3042.10                &4038    & 20734    &0.15      &0\\
    16 & 2    &360 & 2160 &$\infty\quad$            &3328    & 17000    &0.14      &6\\
    15 & 2    &315 & 1785 &$\infty\quad$             &2707    & 13745    &0.37      &66\\
    14 & 115  &273 & 1456 &$\infty\quad$             &2169    & 10936    &~~259.56   &~~24635\\
    13 & ?~    &234 & 1170 &$\infty\quad$             &1708    & 8540     &$\infty\quad$ &$?~$\\
    \hline
  \end{tabular}}
  \caption{The search for $(3,3,3;n)$ Ramsey colorings with
    and without the symmetry break defined 
    in~\cite{CodishMPS14} 
    (time in seconds with 24 hr. timeout).}
\label{tab:333n1}
\end{table}

Table~\ref{tab:333n1} illustrates the impact of the symmetry breaking
technique introduced by Codish \etal~in \cite{CodishMPS14} on the
search for $(3,3,3;n)$ Ramsey colorings.  The column headed by
``\#${\setminus}_{\approx}$'' specifies the known number of colorings
modulo weak isomorphism~\cite{Rad}.
The columns headed by ``\#vars'' and ``\#clauses'' indicate,
respectively, the number of variables and clauses in the corresponding
CNF encodings of the coloring problems with and without the symmetry
breaking constraint.  The columns headed by ``time'' indicate the time
(in seconds, on a single thread of the cluster) to find all colorings
iterating with a SAT solver. The timeout assumed here is 24 hours. The
column headed by ``\#'' specifies the number of colorings found when
solving with the symmetry break. These include colorings which are
weakly isomorphic, but far fewer than the hundreds of thousands
generated without the symmetry break (until the timeout).  The results
in this table were obtained using the CryptoMiniSAT
solver~\cite{Crypto}.

\begin{figure}
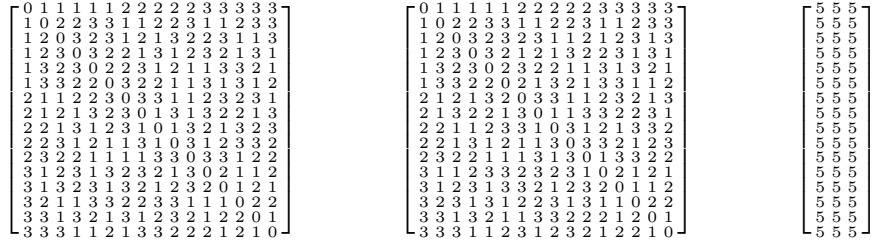

  \begin{center}
$\left[                                  
\begin{smallmatrix}
0 &1 &1 &1 &1 &1 &2 &2 &2 &2 &2 &3 &3 &3 &3 &3 \\
1 &0 &2 &2 &3 &3 &1 &1 &2 &2 &3 &1 &1 &2 &3 &3 \\
1 &2 &0 &3 &2 &3 &1 &2 &1 &3 &2 &2 &3 &1 &1 &3 \\
1 &2 &3 &0 &3 &2 &2 &1 &3 &1 &2 &3 &2 &1 &3 &1 \\
1 &3 &2 &3 &0 &2 &2 &3 &1 &2 &1 &1 &3 &3 &2 &1 \\
1 &3 &3 &2 &2 &0 &3 &2 &2 &1 &1 &3 &1 &3 &1 &2 \\
2 &1 &1 &2 &2 &3 &0 &3 &3 &1 &1 &2 &3 &2 &3 &1 \\
2 &1 &2 &1 &3 &2 &3 &0 &1 &3 &1 &3 &2 &2 &1 &3 \\
2 &2 &1 &3 &1 &2 &3 &1 &0 &1 &3 &2 &1 &3 &2 &3 \\
2 &2 &3 &1 &2 &1 &1 &3 &1 &0 &3 &1 &2 &3 &3 &2 \\
2 &3 &2 &2 &1 &1 &1 &1 &3 &3 &0 &3 &3 &1 &2 &2 \\
3 &1 &2 &3 &1 &3 &2 &3 &2 &1 &3 &0 &2 &1 &1 &2 \\
3 &1 &3 &2 &3 &1 &3 &2 &1 &2 &3 &2 &0 &1 &2 &1 \\
3 &2 &1 &1 &3 &3 &2 &2 &3 &3 &1 &1 &1 &0 &2 &2 \\
3 &3 &1 &3 &2 &1 &3 &1 &2 &3 &2 &1 &2 &2 &0 &1 \\
3 &3 &3 &1 &1 &2 &1 &3 &3 &2 &2 &2 &1 &2 &1 &0 
\end{smallmatrix}\right]$
\qquad\qquad
$\left[                              
\begin{smallmatrix}
0 &1 &1 &1 &1 &1 &2 &2 &2 &2 &2 &3 &3 &3 &3 &3 \\
1 &0 &2 &2 &3 &3 &1 &1 &2 &2 &3 &1 &1 &2 &3 &3 \\
1 &2 &0 &3 &2 &3 &2 &3 &1 &1 &2 &1 &2 &3 &1 &3 \\
1 &2 &3 &0 &3 &2 &1 &2 &1 &3 &2 &2 &3 &1 &3 &1 \\
1 &3 &2 &3 &0 &2 &3 &2 &2 &1 &1 &3 &1 &3 &2 &1 \\
1 &3 &3 &2 &2 &0 &2 &1 &3 &2 &1 &3 &3 &1 &1 &2 \\
2 &1 &2 &1 &3 &2 &0 &3 &3 &1 &1 &2 &3 &2 &1 &3 \\
2 &1 &3 &2 &2 &1 &3 &0 &1 &1 &3 &3 &2 &2 &3 &1 \\
2 &2 &1 &1 &2 &3 &3 &1 &0 &3 &1 &2 &1 &3 &3 &2 \\
2 &2 &1 &3 &1 &2 &1 &1 &3 &0 &3 &3 &2 &1 &2 &3 \\
2 &3 &2 &2 &1 &1 &1 &3 &1 &3 &0 &1 &3 &3 &2 &2 \\
3 &1 &1 &2 &3 &3 &2 &3 &2 &3 &1 &0 &2 &1 &2 &1 \\
3 &1 &2 &3 &1 &3 &3 &2 &1 &2 &3 &2 &0 &1 &1 &2 \\
3 &2 &3 &1 &3 &1 &2 &2 &3 &1 &3 &1 &1 &0 &2 &2 \\
3 &3 &1 &3 &2 &1 &1 &3 &3 &2 &2 &2 &1 &2 &0 &1 \\
3 &3 &3 &1 &1 &2 &3 &1 &2 &3 &2 &1 &2 &2 &1 &0
\end{smallmatrix} \right]$
\qquad\qquad
$\left[                              
\begin{smallmatrix}
5 & 5 & 5 \\
5 & 5 & 5 \\
5 & 5 & 5 \\
5 & 5 & 5 \\
5 & 5 & 5 \\
5 & 5 & 5 \\
5 & 5 & 5 \\
5 & 5 & 5 \\
5 & 5 & 5 \\
5 & 5 & 5 \\
5 & 5 & 5 \\
5 & 5 & 5 \\
5 & 5 & 5 \\
5 & 5 & 5 \\
5 & 5 & 5 \\
5 & 5 & 5 
\end{smallmatrix} \right]$

\label{333_16}
\caption{Two $16\times 16$ non-isomorphic $(3,3,3;16)$ Ramsey
  colorings (left and middle) with their common $16\times 3$ degree
  matrix (right). }
        \label{fig:r333_16graphs}
    
  \end{center}
\end{figure}

Figure~\ref{333_16} depicts, on the left and in the middle, the two
non-isomorphic colorings $(3,3,3;16)$ represented as adjacency graphs
in the form found using the approach of Codish
\etal~\cite{CodishMPS14}. Note the lexicographic order on the rows in
both matrices. These graphs are isomorphic to the two colorings
reported in 1968 by Kalbfleish and Stanton~\cite{KalbfleischStanton68}
where it is also proven that there are no others (modulo weak
isomorphism). The $16\times 3$ degree matrix (right) describes
the degrees of each node in each color as defined below in
Definition~\ref{def:dm}. 
The results reported in Table~\ref{tab:333n1} also illustrate that the
approach of Codish \etal~is not sufficiently powerful to compute the
number of $(3,3,3;13)$ colorings. Likewise, it does not facilitate the
computation of $R(4,3,3)$.

In the following we make use of the following
results from \cite{PR98}.

\begin{theorem}\label{thm:433}
  $30\leq R(4,3,3)\leq 31$ and, $R(4,3,3)=31$ if and only if there
  exists a $(4,3,3;30)$ coloring $\kappa$ of $K_{30}$ such that:
  (1) For every vertex $v$ and $i\in\{2,3\}$, $5\leq deg_{i}(v)\leq
  8$, and $13\leq deg_{1}(v)\leq 16$.
  (2) Every edge in the third color has at least one endpoint $v$ with
  $deg_{3}(v)=13$. 
  (3) There are at least 25 vertices $v$ for which 
  $deg_{1}(v)=13$, $deg_{2}(v)=deg_{3}(v)=8$.
\end{theorem}

The following is a  direct consequence of Theorem~\ref{thm:433}.
\begin{corollary}\label{cor:degrees}
  If $G$ is a $(4,3,3;30)$ coloring, and assuming without loss of
  generality that the degree of color two is greater equal to the
  degree of color three, then every vertex in $G$ has degrees in the
  corresponding colors corresponding to one of the triplets $(13, 8,
  8)$, $(14, 8, 7)$, $(15, 7, 7)$, $(15, 8, 6)$, $(16, 7, 6)$, $(16,
  8, 5)$.
\end{corollary}

Consider a vertex $v$ in a $(4,3,3;n)$ coloring and focus on the three
subgraphs induced by the neighbors of $v$ in each of the three
colors. The following states that these must be corresponding Ramsey
colorings.

\begin{corollary}\label{cor:embed}
  Let $G$ is a $(4,3,3;n)$ coloring and $v$ be any vertex with degrees
  $(d_1,d_2,d_3)$ in the corresponding colors. Then, $d_1+d_2+d_3=n-1$
  and $G^1_v$, $G^2_v$, and $G^3_v$ are respectively $(3,3,3;d_1)$,
  $(4,2,3;d_2)$, and $(4,3,2;d_3)$ colorings.
\end{corollary}

Note that by definition a $(4,2,3;n)$ coloring is a $(4,3;n)$ coloring
in colors 1 and 3 and likewise a $(4,3,2;n)$ coloring is a $(4,3;n)$
coloring in colors 1 and 2.
For $n\in\{14,15,16\}$, the set of all
$(3,3,3;n)$ colorings modulo (weak) isomorphism are known and consist
respectively of 2, 2 and 15 colorings. Similary, for $n\in\{5,6,7,8\}$
the set of all $(4,3;n)$ colorings modulo (weak) isomorphism are known
and consist respectively of 9, 15, 9, and 3 Ramsey colorings.

\section{Searching for Ramsey Colorings with Embeddings}

In this section we apply a general approach where, when seeking a
$(r_1,\ldots,r_k;n)$ Ramsey coloring one selects a ``preferred''
vertex, call it $v_1$, and based on its degrees in each of the $k$
colors, embeds $k$ subgraphs which are corresponding smaller
colorings. Using this approach, we apply
Corollaries~\ref{cor:degrees} and~\ref{cor:embed} to establish that a
$(4,3,3;30)$ coloring, if one exists, must be $(13,8,8)$
regular. Specifically, all vertices have 13 neighbors by way of edges in the
first color and 8 neighbors each, by way of edges in the second and
third colors.

\begin{theorem}\label{thm:regular}
  Any $(4,3,3;30)$ coloring, if one exists, is $(13,8,8)$ regular.
\end{theorem}

\begin{proof}
  By computation as described in the rest of this section.
\end{proof}

We seek a $(4,3,3;30)$ coloring of $K_{30}$, represented as a
$30\times 30$ adjacency matrix $A$.  We focus on the degrees,
$(d_1,d_2,d_3)$ in each of the three colors, of the vertex $v_1$,
corresponding to the first row in $A$, as prescribed by
Corollary~\ref{cor:degrees}.  For each such degree triplet, except for
the case $(13,8,8)$, we take each of the known corresponding colorings
for the subgraphs $G^1_{v_1}$, $G^2_{v_1}$, and $G^3_{v_1}$ and embed
them in $A$. We then apply a SAT solver, to complete the remaining
cells in $A$ to satisfy Constraint~\ref{constraint:r334} of
Figure~\ref{fig:gcp}. If the SAT solver fails, then no such completion
exists. 

To illustrate the approach, consider the case where $v_1$ has degrees
$(14,8,7)$ in the three colors. Figure~\ref{embed_14_8_7} details one
of the embeddings corresponding to this case.  The first row of $A$
specifes the colors of the edges of the 29 neighbors of $v_1$. The
symbol ``$\_$'' indicates an integer variable that takes a value
between 1 and 3.
The neighbors of $v_1$ in color~1 form a submatrix of $A$ embedded in
rows (and columns) 2--15 of the matrix in the Figure.  By
Corollary~\ref{cor:embed} these are a $(3,3,3;14)$ Ramsey coloring and
there are 115 possible such colorings modulo weak isonmorphism. The
Figure details one of them.
Similarly, there are 3 possible subgraphs for the neighbors of $v_1$
in color~2,  (the $3$ $(4,2,3;8)$ colorings). In
Figure~\ref{embed_14_8_7}, rows (and columns) 16--23 detail one such
coloring.
Finally, there are 9 possible subgraphs for the neighbors of $v_1$ in
color~3, (the $9$ $(4,3,2;7)$ colorings). In
Figure~\ref{embed_14_8_7}, rows (and columns) 24--30 detail one such
coloring.

\newcommand{\undA}{\_&\_&\_&\_&\_&\_&\_&\_&\_&\_&\_&\_&\_&\_&\_}
\newcommand{\undB}{\_&\_&\_&\_&\_&\_&\_&\_&\_&\_&\_&\_&\_&\_}
\newcommand{\undC}{\_&\_&\_&\_&\_&\_&\_}
\newcommand{\undD}{\_&\_&\_&\_&\_&\_&\_&\_&\_&\_&\_&\_&\_&\_&\_&\_&\_&\_&\_&\_&\_&\_}
\begin{figure}
  \begin{center}
\resizebox{.49\linewidth}{!}{$\left[                                  
\begin{smallmatrix}
0&1&1&1&1&1&1&1&1&1&1&1&1&1&1&2&2&2&2&2&2&2&2&3&3&3&3&3&3&3&\\
1& 0&1&1&1&1&2&2&2&2&3&3&3&3&3&\undA \\
1& 1&0&2&2&3&1&1&3&3&1&2&2&3&3&\undA \\
1& 1&2&0&3&2&1&2&1&3&1&1&3&2&3&\undA \\
1& 1&2&3&0&2&2&1&3&1&1&3&1&3&2&\undA \\
1& 1&3&2&2&0&1&1&2&2&2&3&3&1&1&\undA \\
1& 2&1&1&2&1&0&3&3&1&2&3&1&2&3&\undA \\
1& 2&1&2&1&1&3&0&1&3&2&1&3&3&2&\undA \\
1& 2&3&1&3&2&3&1&0&1&3&2&3&2&1&\undA \\
1& 2&3&3&1&2&1&3&1&0&3&3&2&1&2&\undA \\
1& 3&1&1&1&2&2&2&3&3&0&2&2&1&1&\undA \\
1& 3&2&1&3&3&3&1&2&3&2&0&1&1&2&\undA \\
1& 3&2&3&1&3&1&3&3&2&2&1&0&2&1&\undA \\
1& 3&3&2&3&1&2&3&2&1&1&1&2&0&2&\undA \\
1& 3&3&3&2&1&3&2&1&2&1&2&1&2&0&\undA \\
2& \undB&0&2&2&2&1&1&1&1&\undC \\
2& \undB&2&0&1&1&2&2&1&1&\undC \\
2& \undB&2&1&0&1&2&1&2&1&\undC \\
2& \undB&2&1&1&0&1&1&2&2&\undC \\
2& \undB&1&2&2&1&0&1&1&2&\undC \\
2& \undB&1&2&1&1&1&0&2&2&\undC \\
2& \undB&1&1&2&2&1&2&0&1&\undC \\
2& \undB&1&1&1&2&2&2&1&0&\undC \\
3&    \undD& 0&2&2&1&1&1&1 \\
3&    \undD& 2&0&1&2&1&1&1 \\
3&    \undD& 2&1&0&1&2&1&1 \\
3&    \undD& 1&2&1&0&1&2&1 \\
3&    \undD& 1&1&2&1&0&1&2 \\
3&    \undD& 1&1&1&2&1&0&2 \\
3&    \undD& 1&1&1&1&2&2&0 \\                   
\end{smallmatrix}\right]$}
\caption{One of 3105 embeddings in the search for a $(4,3,3;30)$
  coloring when the first vertex has degrees $(14,8,7)$.   }
        \label{embed_14_8_7}
  \end{center}
\end{figure}

To summarize, Figure~\ref{embed_14_8_7} is a partial instantiated
adjacency matrix in which the first row determines the degrees of
$v_1$ in each of the three colors, and where 3 corresponding subgraphs
are embedded.  The uninstantiated values in the matrix must be
completed to obtain a solution that satisfies
Constraint~\ref{constraint:r334} of Figure~\ref{fig:gcp}. This can be
determined using a SAT solver.
%
%
For the specific example in Figure~\ref{embed_14_8_7}, the CNF
generated using our tool set consists of 33{,}959 clauses, involves
5{,}318 Boolean variables, and is shown to be unsatisfiable in 52 seconds of
computation time.
For the case where $v_1$ has degrees $(14,8,7)$ in the three colors
this is one of $115\times 3\times 9 = 3105$ instances that need to be
checked.

Table~\ref{table:regular} summarizes the experiment which proves
Theorem~\ref{thm:regular}. For each of the possible degrees of
vertex~1 in a $(4,3,3;30)$ coloring as prescribed
by Corollary~\ref{cor:degrees}, except $(13,8,8)$, and for each possible choice of
colorings for the derived subgraphs $G^1_{v_1}$, $G^2_{v_1}$, and
$G^3_{v_1}$, we apply a SAT solver to show that
Constraint~\ref{constraint:r334} of Figure~\ref{fig:gcp} cannot be
satisfied. The table details for each degree triple, the number of
instances, their average size (number of clauses and Boolean
variables), and the average and total times to show that the
constraint is not satisfieable.

\begin{table}
\begin{center}
\begin{tabular}{|c|c|c|c|r|r|}
\hline
$v_1$ degrees    & \# instances & \# clauses (avg.) & \# vars (avg.) &
unsat (avg).) & unsat (total) \\ \hline
(16,8,5) & 54 (2*3*9)   & 32432 & 5279 &  51 sec. &  0.77 hrs.\\ 
\hline
(16,7,6) & 270 (2*9*15) & 32460 & 5233 & 420 sec. & 31.50 hrs.\\ 
\hline
(15,8,6) & 90 (2*3*15)  & 33607 & 5450 &  93 sec. &  2.32 hrs.\\ 
\hline
(15,7,7) & 162 (2*9*9)  & 33340 & 5326 &1554 sec. & 69.94 hrs.\\ 
\hline
(14,8,7) &3105 (115*3*9)& 34069 & 5324 & 294 sec. &253.40 hrs.\\ 
\hline

\end{tabular}
\end{center}
\caption{Proving that any $(4,3,3;30)$ Ramsey coloring is $(13,8,8)$
  regular (summary).}
\label{table:regular}
\end{table}

To gain confidence in our implementation, we illustrate its
application to find a $(4,3,3;29)$ coloring which is known to exist.
This experiment involves some reverse engineering.
In 1966 Kalbfleisch~\cite{kalb66} reported the existence of a
circulant $(3,4,4;29)$ coloring. Encoding
Constraint~\ref{constraint:r334} with $n=29$, together with a
constraint that states that the adjacency matrix $A$ is circulant,
results in a CNF with 146{,}506 clauses and 8{,}394 variables. Using a
SAT solver, we obtain a corresponding $(4,3,3;29)$ coloring in less
than two seconds of computation time. The solution is $(12,8,8)$ regular
and permuting its first row to be of the form
$01111111111112222222233333333$ we extract from it three corresponding
subgraphs: $G^1_{v_1}$, $G^2_{v_1}$ and $G^3_{v_1}$ which are
respectively $(3,3,3;12)$, $(4,2,3;8)$ and $(4,3,2;8)$ Ramsey colorings. An
embedding of these three in a $29\times 29$ adjacency matrix is
depicted as Figure~\ref{embed_12_8_8}.

\begin{figure}
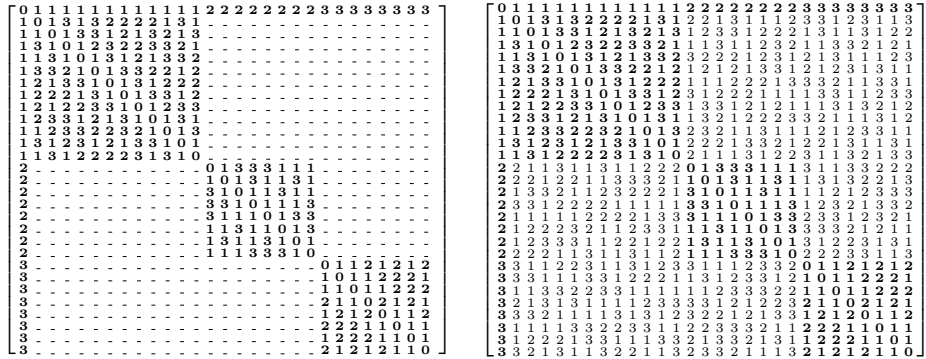
  
\newcommand{\uuA}{\_&\_&\_&\_&\_&\_&\_&\_&\_&\_&\_&\_&\_&\_&\_&\_}
\newcommand{\uuB}{\_&\_&\_&\_&\_&\_&\_&\_&\_&\_&\_&\_}
\newcommand{\uuC}{\_&\_&\_&\_&\_&\_&\_&\_}
\newcommand{\uuD}{\_&\_&\_&\_&\_&\_&\_&\_&\_&\_&\_&\_&\_&\_&\_&\_&\_&\_&\_&\_}
\newcommand{\bz}{{\bf 0}} 
\newcommand{\ba}{{\bf 1}} 
\newcommand{\bb}{{\bf 2}} 
\newcommand{\bc}{{\bf 3}}  
\resizebox{.49\linewidth}{!}{$\left[                                  
\begin{smallmatrix}
\bz& \ba&\ba&\ba&\ba&\ba&\ba&\ba&\ba&\ba&\ba&\ba&\ba&\bb&\bb&\bb&\bb&\bb&\bb&\bb&\bb&\bc&\bc&\bc&\bc&\bc&\bc&\bc&\bc& \\
\ba& \bz&\ba&\bc&\ba&\bc&\bb&\bb&\bb&\bb&\ba&\bc&\ba &\uuA \\
\ba& \ba&\bz&\ba&\bc&\bc&\ba&\bb&\ba&\bc&\bb&\ba&\bc &\uuA \\
\ba& \bc&\ba&\bz&\ba&\bb&\bc&\bb&\bb&\bc&\bc&\bb&\ba &\uuA \\
\ba& \ba&\bc&\ba&\bz&\ba&\bc&\ba&\bb&\ba&\bc&\bc&\bb &\uuA \\
\ba& \bc&\bc&\bb&\ba&\bz&\ba&\bc&\bc&\bb&\bb&\ba&\bb &\uuA \\
\ba& \bb&\ba&\bc&\bc&\ba&\bz&\ba&\bc&\ba&\bb&\bb&\bb &\uuA \\
\ba& \bb&\bb&\bb&\ba&\bc&\ba&\bz&\ba&\bc&\bc&\ba&\bb &\uuA \\
\ba& \bb&\ba&\bb&\bb&\bc&\bc&\ba&\bz&\ba&\bb&\bc&\bc &\uuA \\
\ba& \bb&\bc&\bc&\ba&\bb&\ba&\bc&\ba&\bz&\ba&\bc&\ba &\uuA \\
\ba& \ba&\bb&\bc&\bc&\bb&\bb&\bc&\bb&\ba&\bz&\ba&\bc &\uuA \\
\ba& \bc&\ba&\bb&\bc&\ba&\bb&\ba&\bc&\bc&\ba&\bz&\ba &\uuA \\
\ba& \ba&\bc&\ba&\bb&\bb&\bb&\bb&\bc&\ba&\bc&\ba&\bz &\uuA \\
\bb& \uuB&     \bz&\ba&\bc&\bc&\bc&\ba&\ba&\ba& \uuC \\
\bb& \uuB&     \ba&\bz&\ba&\bc&\ba&\ba&\bc&\ba& \uuC \\
\bb& \uuB&     \bc&\ba&\bz&\ba&\ba&\bc&\ba&\ba& \uuC \\
\bb& \uuB&     \bc&\bc&\ba&\bz&\ba&\ba&\ba&\bc& \uuC \\
\bb& \uuB&     \bc&\ba&\ba&\ba&\bz&\ba&\bc&\bc& \uuC \\
\bb& \uuB&     \ba&\ba&\bc&\ba&\ba&\bz&\ba&\bc& \uuC \\
\bb& \uuB&     \ba&\bc&\ba&\ba&\bc&\ba&\bz&\ba& \uuC \\
\bb& \uuB&     \ba&\ba&\ba&\bc&\bc&\bc&\ba&\bz& \uuC \\
\bc&    \uuD& \bz&\ba&\ba&\bb&\ba&\bb&\ba&\bb \\
\bc&    \uuD& \ba&\bz&\ba&\ba&\bb&\bb&\bb&\ba \\
\bc&    \uuD& \ba&\ba&\bz&\ba&\ba&\bb&\bb&\bb \\
\bc&    \uuD& \bb&\ba&\ba&\bz&\bb&\ba&\bb&\ba \\
\bc&    \uuD& \ba&\bb&\ba&\bb&\bz&\ba&\ba&\bb \\
\bc&    \uuD& \bb&\bb&\bb&\ba&\ba&\bz&\ba&\ba \\
\bc&    \uuD& \ba&\bb&\bb&\bb&\ba&\ba&\bz&\ba \\
\bc&    \uuD& \bb&\ba&\bb&\ba&\bb&\ba&\ba&\bz 
\end{smallmatrix}\right]$}
\quad
\resizebox{.49\linewidth}{!}{$\left[  
\begin{smallmatrix}
\bz&\ba&\ba&\ba&\ba&\ba&\ba&\ba&\ba&\ba&\ba&\ba&\ba&\bb&\bb&\bb&\bb&\bb&\bb&\bb&\bb&\bc&\bc&\bc&\bc&\bc&\bc&\bc&\bc\\
\ba&\bz&\ba&\bc&\ba&\bc&\bb&\bb&\bb&\bb&\ba&\bc&\ba&2&2&1&3&1&1&1&2&3&3&1&2&3&1&1&3\\
\ba&\ba&\bz&\ba&\bc&\bc&\ba&\bb&\ba&\bc&\bb&\ba&\bc&1&2&3&3&1&2&2&2&1&3&1&1&3&1&2&2\\
\ba&\bc&\ba&\bz&\ba&\bb&\bc&\bb&\bb&\bc&\bc&\bb&\ba&1&1&3&1&1&2&3&2&1&1&3&3&2&1&2&1\\
\ba&\ba&\bc&\ba&\bz&\ba&\bc&\ba&\bb&\ba&\bc&\bc&\bb&3&2&2&2&1&2&3&1&2&1&3&1&1&1&2&3\\
\ba&\bc&\bc&\bb&\ba&\bz&\ba&\bc&\bc&\bb&\bb&\ba&\bb&1&2&1&2&1&3&3&1&2&1&2&3&1&3&1&1\\
\ba&\bb&\ba&\bc&\bc&\ba&\bz&\ba&\bc&\ba&\bb&\bb&\bb&1&1&1&2&2&2&1&3&3&3&2&1&1&3&3&1\\
\ba&\bb&\bb&\bb&\ba&\bc&\ba&\bz&\ba&\bc&\bc&\ba&\bb&3&1&2&2&2&1&1&1&1&3&3&1&1&2&3&3\\
\ba&\bb&\ba&\bb&\bb&\bc&\bc&\ba&\bz&\ba&\bb&\bc&\bc&1&3&3&1&2&1&2&1&1&1&3&1&3&2&1&2\\
\ba&\bb&\bc&\bc&\ba&\bb&\ba&\bc&\ba&\bz&\ba&\bc&\ba&1&3&2&1&2&2&2&3&3&2&1&1&1&3&1&2\\
\ba&\ba&\bb&\bc&\bc&\bb&\bb&\bc&\bb&\ba&\bz&\ba&\bc&2&3&2&1&1&3&1&1&1&2&1&2&3&3&1&1\\
\ba&\bc&\ba&\bb&\bc&\ba&\bb&\ba&\bc&\bc&\ba&\bz&\ba&2&2&2&1&3&3&2&1&2&2&1&3&1&1&3&1\\
\ba&\ba&\bc&\ba&\bb&\bb&\bb&\bb&\bc&\ba&\bc&\ba&\bz&2&1&1&1&3&1&2&2&3&1&1&3&2&1&3&3\\
\bb&2&1&1&3&1&1&3&1&1&2&2&2&\bz&\ba&\bc&\bc&\bc&\ba&\ba&\ba&3&1&1&3&3&2&2&2\\
\bb&2&2&1&2&2&1&1&3&3&3&2&1&\ba&\bz&\ba&\bc&\ba&\ba&\bc&\ba&1&3&1&3&2&2&1&3\\
\bb&1&3&3&2&1&1&2&3&2&2&2&1&\bc&\ba&\bz&\ba&\ba&\bc&\ba&\ba&1&1&2&1&2&3&3&3\\
\bb&3&3&1&2&2&2&2&1&1&1&1&1&\bc&\bc&\ba&\bz&\ba&\ba&\ba&\bc&1&2&3&2&1&3&3&2\\
\bb&1&1&1&1&1&2&2&2&2&1&3&3&\bc&\ba&\ba&\ba&\bz&\ba&\bc&\bc&2&3&3&1&2&3&2&1\\
\bb&1&2&2&2&3&2&1&1&2&3&3&1&\ba&\ba&\bc&\ba&\ba&\bz&\ba&\bc&3&3&3&2&1&2&1&1\\
\bb&1&2&3&3&3&1&1&2&2&1&2&2&\ba&\bc&\ba&\ba&\bc&\ba&\bz&\ba&3&1&2&2&3&1&3&1\\
\bb&2&2&2&1&1&3&1&1&3&1&1&2&\ba&\ba&\ba&\bc&\bc&\bc&\ba&\bz&2&2&2&3&3&1&1&3\\
\bc&3&1&1&2&2&3&1&1&3&1&2&3&3&1&1&1&2&3&3&2&\bz&\ba&\ba&\bb&\ba&\bb&\ba&\bb\\
\bc&3&3&1&1&1&3&3&1&2&2&2&1&1&3&1&2&3&3&1&2&\ba&\bz&\ba&\ba&\bb&\bb&\bb&\ba\\
\bc&1&1&3&3&2&2&3&3&1&1&1&1&1&1&2&3&3&3&2&2&\ba&\ba&\bz&\ba&\ba&\bb&\bb&\bb\\
\bc&2&1&3&1&3&1&1&1&1&2&3&3&3&3&1&2&1&2&2&3&\bb&\ba&\ba&\bz&\bb&\ba&\bb&\ba\\
\bc&3&3&2&1&1&1&1&3&1&3&1&2&3&2&2&1&2&1&3&3&\ba&\bb&\ba&\bb&\bz&\ba&\ba&\bb\\
\bc&1&1&1&1&3&3&2&2&3&3&1&1&2&2&3&3&3&2&1&1&\bb&\bb&\bb&\ba&\ba&\bz&\ba&\ba\\
\bc&1&2&2&2&1&3&3&1&1&1&3&3&2&1&3&3&2&1&3&1&\ba&\bb&\bb&\bb&\ba&\ba&\bz&\ba\\
\bc&3&2&1&3&1&1&3&2&2&1&1&3&2&3&3&2&1&1&1&3&\bb&\ba&\bb&\ba&\bb&\ba&\ba&\bz
\end{smallmatrix}\right]$}

\caption{Embedding (left) and solution (right) for $(3,3,4;29)$ Ramsey coloring}
        \label{embed_12_8_8}

\end{figure}

Applying a SAT solver to complete this embedding to a $(4,3,3;29)$
coloring that satsifies Constraint~~\ref{constraint:r334} involves a
CNF with 30{,}944 clauses and 4{,}736 variables and requires under two
minutes of computation time.


\paragraph{Proving that $R(4,3,3)=30$.}


To apply the embedding approach described in this section to prove
that there is no $(4,3,3;30)$ Ramsey coloring which is $(13,8,8)$
regular would require considering all $(3,3,3;13)$ colorings modulo
weak isomorphism. Doing this would constitute a proof that
$R(4,3,3)=30$.  We defer this discussion until after
Section~\ref{sec:dm2} where we describe how we compute the set of all
78{,}892 $(3,3,3;13)$ Ramsey colorings modulo weak isomorphism.







\section{Degree Matrices for Graph Colorings}

We introduce an abstraction on graph colorings defined in terms of
\emph{degree matrices} and an equivalence relation on degree
matrices. Our motivation is to solve graph coloring problems by first
focusing on an over approximation of their degree matrices. The
equivalence relation on degree matrices enables us to break symmetries
during search when solving graph coloring problems.
Intuitively, degree matrices are to graph edge-colorings as
degree sequences are to graphs.

\begin{definition}[\textbf{abstraction, degree matrix}]
\label{def:dm}
  Let $A$ be a graph coloring on $n$ vertices with $k$ colors. The
  \emph{degree matrix} of $A$, denoted $\alpha(A)$ is an $n\times k$
  matrix, $M$ such that $M_{i,j} = deg_j(i)$ is the degree of vertex
  $i$ in color $j$. For a set $\AA$ of graph colorings we denote
  $\alpha(\AA) = \sset{\alpha(A)}{A\in\AA}$.
\end{definition}

A degree matrix, $M$, is said to {\it represent} the set of graphs
weakly-isomorphic to a graph with degrees as in $M$. We say that two
degree matrices are equivalent if they represent the same sets of
graph colorings.

\begin{definition}[\textbf{concretization and equivalence}]
\label{def:conc}
  Let $M$ and $N$ be $n\times k$ degree matrices. Then, $\gamma(M) =
  \sset{A}{A\approx A',~\alpha(A')=M}$ is the set of graph colorings
  represented by $M$ and we say that $M\equiv N \Leftrightarrow
  \gamma(M)=\gamma(N)$. For a set $\MM$ of degree matrices we denote
  $\gamma(\MM) = \cup\sset{\gamma(M)}{M\in\MM}$.
\end{definition}

 Due to properties of weak-isomorphism (vertices as well as colors can
 be reordered) we can exchange both rows and columns of a degree matrix
 without changing the set of graphs it represents.
 In our construction we assume that the rows and columns of a degree
 matrix are sorted lexicographically. Observe also that the columns of
 a degree matrix each form a graphic sequence (when sorted).

\begin{definition}[\textbf{lex sorted degree matrix}]
  For an $n\times k$ degree matrix $M$ we denote by $lex(M)$ the smallest
  matrix with rows and columns in the lexicographic order
  (non-increasing) obtained by permuting rows and columns of $M$.
\end{definition}

The following  implies that for degree matrices we can assume
without loss of generality that rows and columns are lexicographically ordered.

\begin{theorem}\label{thm:dme}
  If $M$, $N$ are degree matrices then $M\equiv N$ if and only if
  there exists permutations $\pi \colon [n] \to [n]$ and $\sigma
  \colon [k] \to [k]$ such that, for $1\leq i \leq n$, $1\leq j\leq
  k$, $M_{i,j} = N_{\pi(i),\sigma(j)}$.
\end{theorem}

\begin{proof}
  Let $M$ and $N$ be degree matrices. Then,
  \[\begin{array}{l}
    M\equiv N   \xLeftrightarrow{\mbox{\tiny{~Defn.~\ref{def:conc}~}}}
    \gamma(M)=\gamma(N)  \xLeftrightarrow{\mbox{\tiny ~Defn.~\ref{def:conc}~}}
    \forall_{G\approx H}.G\in\gamma(M)\leftrightarrow H\in\gamma(N) 
                    \xLeftrightarrow{\mbox{\tiny ~Defn.~\ref{def:weak_iso}~}}\\
    \exists_{\pi,\sigma}.\alpha(G)_{i,j}=\alpha(H)_{\pi(i),\sigma(j)} 
                         \xLeftrightarrow{\mbox{\tiny ~Defn.~\ref{def:dm}~}}
    M_{i,j} = N_{\pi(i),\sigma(j)} 
  \end{array}\]
  
\end{proof}

\comment{
\begin{proof}
  Let $A$ be a matrix for which $\alpha(A)=M$. Then for some
  $A'\approx A$, $\alpha(A')=N$.  Let $\pi \colon [n] \to [n]$ and
  $\sigma \colon [k] \to [k]$ be the permutations such that
  Definition~\ref{def:weak_iso} holds.
  Now
\begin{eqnarray*}\small
\alpha(A)_{i,j}&=& \big|\sset{A_{i,m}}{A_{i,m}=j}\big|\\
      &=& \big| \sset{A'_{\pi(i),\pi(m)}}{A'_{\pi(i),\pi(m)}=\sigma(j))} \big|\\
      &=& \alpha(A')_{\pi(i),\sigma(j)}
\end{eqnarray*}
and the result follows.
\end{proof}
}

\begin{corollary}\label{cor:lexM}
 $M \equiv lex(M)$.
\end{corollary}
\begin{proof}
  The result follows from Theorem~\ref{thm:dme} because $M$ and
  $lex(M)$ are related by permutations of rows and columns.
\end{proof}

\begin{example}
  The degree matrix on the right of Figure~\ref{333_16} describes
  both of the graphs in the figure.
\end{example}

\section{Solving Graph Coloring Problems with Degree Matrices}
\label{sec:adm}

Let $\varphi(A)$ be a graph coloring problem in $k$ colors on an
$n\times n$ adjacency matrix, $A$.
Assuming that $\AA=sol(\varphi(A))$ is too hard to compute, either
because the number of solutions is too large or because finding even
a single solution is too hard, 
our strategy is to first compute an over-approximation  $\MM$ of
degree matrices such that $\gamma(\MM)\supseteq\AA$ and to then use
$\MM$ to guide the computation of $\AA$. 
We denote the set of solutions of the graph coloring problem,
$\varphi(A)$, which have a given degree matrix, $M$, by
$sol_M(\varphi(A))$ and we have
\begin{equation}\label{eq:solM}
	sol_M(\varphi(A)) = sol(\varphi(A)\wedge\alpha(A){=}M)
\end{equation}
Note that $M\not\in\alpha(sol(\varphi(A)))\Rightarrow
sol_M(\varphi(A))=\emptyset$.  Hence, for $\MM \supseteq
\alpha(sol(\varphi(A)))$,
\begin{equation}\label{eq:approx}
sol(\varphi(A)) = 
    \bigcup_{M\in\MM} sol_M(\varphi(A)) 
\end{equation}
Equation~(\ref{eq:approx}) implies that, using any over-approximation $\MM
\supseteq \alpha(sol(\varphi(A)))$, we can compute the solutions to a
graph coloring problem by computing the independent sets
$sol_M(\varphi(A))$ for each $M \in \MM$.
This facilitates the computation of $sol(\varphi(A))$ for three
reasons:
(1) The problem is now broken into a set of independent sub-problems
for each $M\in\MM$ which can be solved in parallel.
(2) The computation of each individual $sol_M(\varphi(A))$ is now
directed using $M$, and
(3) Symmetry breaking is facilitated. 

There are two sources of symmetries when solving $\varphi(A)$.  First,  we compute $\MM$ to consist of canonical degree matrices,
sorted lexicographically by rows and by columns. Second, we impose an
additional symmetry breaking constraint $\SB^*_\ell(A,M)$ as explained
below.

Consider a computation of all solutions of the constraint in the right
side of Equation~(\ref{eq:solM}). Consider a permutation $\pi$ of the
rows and columns of $A$, such that $\alpha(\pi(A))=\alpha(A)=M$. Then,
both $A$ and $A'$ are solutions and they are weakly isomorphic. The
following equation
\begin{equation}
  \label{eq:scenario1}
  sol_M(\varphi(A)) = sol(\varphi(A)\wedge (\alpha(A){=}M) \wedge\SB^*_\ell(A,M))
\end{equation}
refines Equation~(\ref{eq:solM}) introducing a symmetry breaking
constraint similar to the (partitioned lexicographic) symmetry break
predicate introduced by Codish \etal~in
\cite{DBLP:conf/ijcai/CodishMPS13} for Boolean adjacency matrices.
\begin{equation}\label{eq:sbdm}
  \SB^*_\ell(A,M) = 
       \bigwedge_{i<j} \left(\begin{array}{l}
         \big(M_i=M_j\Rightarrow A_i\preceq_{\{i,j\}} A_j\big)
      \end{array}\right)
  \end{equation}
where $s\preceq_{\{i,j\}}s'$ denotes the lexicographic order on
strings $s$ and $s'$ after simultaneously omitting the elements at
positions $i$ and $j$.


To justify that Equations~(\ref{eq:solM}) and~(\ref{eq:scenario1})
both compute $sol_M(\varphi(A))$, modulo weak isomorphism, we must show that
whenever $\SB^*_\ell(A,M)$ excludes a solution then there is another weakly
isomorphic solution that is not excluded. To this end, we introduce a
definition and then a theorem.

\begin{definition}[\textbf{degree matrix preserving permutation}]
  Let $A$ be an adjacency matrix with a lexicographically ordered
  degree matrix $\alpha(A) = M$.  We say that permutation
  $\pi$ is \emph{degree matrix preserving} for $M$ and $A$ if
  $\alpha(\pi(A)) = M$.
\end{definition}

\begin{theorem}[\textbf{correctness of $\SB^{*}_\ell(A,M)$}]
  Let $A$ be an adjacency matrix with a lexicographically ordered
  degree matrix $\alpha(A) = M$. Then, there exists a degree matrix preserving
  permutation $\pi$ such that $\alpha(\pi(A)) = M$ and
  $\SB^{*}_\ell(\pi(A),M)$ holds.
\end{theorem}
\begin{proof}
  If  the rows of $M$ are distinct, then the theorem holds with $\pi$ the
  identity permutation. Assume that some rows of $M$ are equal.
  Denote by $P$ the set of degree matrix preserving permutations for
  $M$ and $A$. Assume the premise and that no $\pi\in
  P$ satisfies $\SB^{*}_\ell(\pi(A),M)$.
  Let $\pi\in P$ be such that $\pi(A) = \min\sset{\pi'(A) \in P}{\pi'
    \in P}$ (in the lexicographical order viewing matrices as
  strings).
  From the assumption, there exist $i<j$ such that $M_i=M_j$ and
  $\pi(A)_i \not\preceq_{\{i,j\}} \pi(A)_j$. Hence there exists a
  minimal index $k\notin\{ i,j\}$ such that $\pi(A)_{i,k} > \pi(A)_{j,k}$.
  Let $A'$ be the matrix obtained by permuting nodes $i$ and $j$ in
  $\pi(A) $. Since $M_i = M_j$ it follows that $\alpha(A') = M$. Thus
  there is a $\pi' \in P$ such that $\pi'(A) = A'$.
  If $k < i$ : for $1\leq l < k$ we have $\pi(A)_l=A'_l$.
  Thus $k$ is the first row for which $A'$ and $\pi(A)$ differ.
  Permuting nodes $i$ and $j$ changes row $k$ by simply swapping elements
  $\pi(A)_{k,i}$ and $\pi(A)_{k,j}$. Since $\pi(A)_{k,i} >
  \pi(A)_{k,j}$, clearly $A'_k \prec \pi(A)_k$ hence $A' \prec
  \pi(A)$ which is a contradiction.
  Similarly if $k > i$ the same argument applies to show that $i$ is
  the first row for which $A'$ and $\pi(A)$ differ, thus obtaining the
  same contradiction for row $i$.
\end{proof}

The following corollary clarifies that if a solution $A$ is eliminated
when introducing the symmetry break predicate
$\SB^{*}_\ell(A,\alpha(A))$ to a graph coloring problem then there
always remains an isomorphic solution $A'$ which satisfies the predicate
$\SB^{*}_\ell(A',\alpha(A'))$.

\begin{corollary}
  Let $A$ be an adjacency matrix. Then there exists $A'$ isomorphic to $A$
  such that $\alpha(A')$ is lex ordered and $\SB^{*}_\ell(A',\alpha(A'))$ holds.
\end{corollary}

\begin{proof}
  Let $M = \alpha(A)$. From Corollary~\ref{cor:lexM} we know that $M \equiv
  lex(M)$, thus there exists $A''$ isomorphic to $A$ such that
  $\alpha(A'') = lex(M)$.  From Theorem 3 it follows that there exists
  a degree matrix preserving permutation $\pi$ such that
  $\alpha(\pi(A'')) = lex(M)$ and
  $\SB^{*}_\ell(\pi(A''),\alpha(\pi(A'')))$ holds.  If $A' = \pi(A'')$
  then $A'$ is isomorphic to $A$, $\alpha(A')$ is lex ordered and
  $\SB^{*}_\ell(A',\alpha(A'))$ holds.
\end{proof}

\section{Computing Degree Matrices for $R(3,3,3;13)$}
\label{sec:adm1}

This section described how we compute a set $\MM$ of degree matrices
that approximate those of the solutions of
Constraint~\ref{constraint:r333}. We apply a strategy in which we mix SAT
solving with brute-force enumeration as follows. The computation of
the degree matrices is summarized in Table~\ref{tab:333_computeDMs}.
In the first step, we compute bounds on the degrees of the nodes in
any $R(3,3,3;13)$ coloring. 

\begin{lemma}\label{lemma:db}
  Let $A$ be a $R(3,3,3;13)$ coloring then for every vertex $x$ in $A$,
  and color $c\in\{1,2,3\}$, $2\leq deg_{c}(x)\leq 5$.
\end{lemma}

\begin{proof}
  By solving Constraint~\ref{constraint:r333} together with 
  $\SB^*_\ell(A,M)$ seeking a graph with minimal
  degree less than 2 or maximal degree greater than 5.
  The CNF encoding is of size 13672 clauses
  with 2748 Boolean variables and takes under 15 seconds to solve
  and yields an UNSAT result which implies that such graph does not exist.
  
\end{proof}

In the second step, we enumerate the degree sequences with values
within the bounds specified by Lemma~\ref{lemma:db}. Recall that the
degree sequence of an undirected graph is the non-increasing sequence
of its vertex degrees. Not every non-increasing sequence of integers
corresponds to a degree sequence. A sequence that corresponds to a degree sequence is said to be graphical. The number of degree sequences of
graphs with 13 vertices is 836{,}315 (see Sequence number
\texttt{A004251} of The On-Line Encyclopedia of Integer Sequences,
published electronically at\url{http://oeis.org}). However, when the
degrees are bound by Lemma~\ref{lemma:db} There are only 280.

\begin{lemma}\label{lemma:ds}
  There are 280 degree sequences with values between $2$ and
  $5$.
\end{lemma}

\begin{proof}
  By straightforward enumeration using the algorithm of Erdos and
  Gallai~\cite{ErdosGallai1960}.
\end{proof}

In the third step, we test each of the 280 degree sequences identified
by Lemma~\ref{lemma:ds} to determine how many of them might occur as
the left column in a degree matrix.

\begin{lemma}\label{lemma:ds2}
  Let $A$ be a $R(3,3,3;13)$ coloring and let $M$ be the canonical
  form of $\alpha(A)$. Then, (a) the left column of $M$ is one of the
  280 degree sequences identified in Lemma~\ref{lemma:ds}; and (b)
  there are only 80 degree sequences from the 280 which are the left
  column of $\alpha(A)$ for some coloring $A$ in $R(3,3,3;13)$.
\end{lemma}

\begin{proof}
  By solving Constraint~\ref{constraint:r333} with each degree
  sequence from Lemma~\ref{lemma:ds} to test if it is satisfiable.
  This involves 280 instances with average CNF size: 10861 clauses and
  2215 Boolean variables. The total solving time is 375.76 hours and
  the hardest instance required about 50 hours. These instances were
  solved in parallel on the cluster described in
  Section~\ref{sec:intro}.
\end{proof}

In the fourth step we extend the 80 degree sequences identified in
Lemma~\ref{lemma:ds2} to obtain all possible degree matrices.

\begin{lemma}\label{lemma:dm}
  Given the 80 degree sequences identified in Lemma~\ref{lemma:ds2} as
  potential left columns of a degree matrix, there are 11{,}933
  possible degree matrices.
\end{lemma}
\begin{proof}
  By straightforward enumeration. The rows and columns are lex sorted,
  must sum to 12, and the columns must be graphical (when sorted).  We
  first compute all of the degree matrices and then select the
  smallest representatives under permutations of rows and columns. The
  computation requires a few seconds.
\end{proof}

In the fifth step, we test each of the 11{,}933 degree matrices
identified by Lemma~\ref{lemma:ds2} to determine how many of them 
are the abstraction of some $R(3,3,3;13)$ coloring. 

\begin{lemma}\label{lemma:dm2}
  From the 11{,}933 degree matrices identified in
  Lemma~\ref{lemma:dm}, 999 are $\alpha(A)$ for a coloring $A$ in
  $R(3,3,3;13)$. 
\end{lemma}

\begin{proof}
  By solving Constraint~\ref{constraint:r333} together with a given
  degree matrix to test if it is satisfiable.  This involves 11{,}933
  instances with average CNF size:  7632 clauses and  1520 Boolean
  variables. The total solving time is 126.55 hours and the hardest
  instance required 0.88 hours. These instances were solved in
  parallel on the cluster described in
  Section~\ref{sec:intro}. 
\end{proof}

\begin{table}[t]
\centering\scriptsize
\begin{tabular}{ |c|l|c|c|c|}
\hline
Step &  \multicolumn{1}{|c|}{Notes}&  
        \multicolumn{1}{|c|}{ComputationTimes} & 
        \multicolumn{2}{|c|}{CNF Size}\\
\hline\hline
\multirow{2}{*}{1}
  & compute degree bounds (Lemma~\ref{lemma:db})  & 
        \multirow{2}{*}{12.52 sec.} & \#Vars & \#Clauses\\
        \cline{4-5}
  & (1 instance, unsat)    &  ~    & \hfill 2748  &\hfill 13672   \\
\hline
\multirow{2}{*}{2}
  & enumerate 280 possible degree sequences (Lemma~\ref{lemma:ds})     & 
         \multicolumn{3}{|c|}{Prolog, fast (seconds)} \\
\hline

\multirow{2}{*}{3}
  & test degree sequences (Lemma~\ref{lemma:ds2}) & 16.32 hrs. & \#Vars & \#Clauses\\
                           \cline{4-5}
  & (280 instances: 200 unsat, 80 sat)    &  hardest: 1.34 hrs    & \hfill 1215 (avg)  &\hfill 7729(avg)   \\
  
\hline

{4}
  & enumerate 11{,}933 degree matrices (Lemma~\ref{lemma:dm})   & \multicolumn{3}{|c|}
                                                {Prolog, fast
                                                  (seconds)} \\
\hline
\multirow{2}{*}{5}
  & test degree matrices (Lemma~\ref{lemma:dm2}) & 126.55 hrs. & \#Vars & \#Clauses\\
                           \cline{4-5}
  & (11{,}933 instances: 10{,}934 unsat, 999 sat)    &  hardest: 0.88 hrs.    & \hfill 1520 (avg)  &\hfill 7632 (avg)   \\
\hline
\hline

\end{tabular}
\caption{Computing the degree matrices for $\RR(3,3,3;13)$ step by step.}
\label{tab:333_computeDMs}
\end{table}

\section{Computing $R(3,3,3;13)$ from Degree Matrices}
\label{sec:dm2}

We describe the computation of the set of all $(3,3,3;13)$ colorings
starting from the 3805 degree matrices identified in
Section~\ref{sec:adm1}. Table~\ref{tab:333_times} summarizes the two
step experiment reporting the computation on three different SAT
solvers: MiniSAT~\cite{minisat,EenS03}, CryptoMiniSAT~\cite{Crypto},
and Glucose~\cite{Glucose,AudemardS09}.

\paragraph{\bf step 1:}
For each degree matrix we compute, using a SAT solver, all
corresponding solutions of Equation~(\ref{eq:scenario1}), where $\varphi(A)$ is
constraint (4) and $M$ is one of the 999 degree matrices identified in (Lemma~\ref{lemma:dm2}).
These instances were solved in parallel on the cluster described in Section~\ref{sec:intro}.
This generates in total 129{,}188 $(3,3,3;13)$ Ramsey colorings.
Table~\ref{tab:333_times} details the total solving time for these
instances and the solving times for the hardest instance for each SAT
solver. The largest number of graphs generated by a single instance is
3720.

\paragraph{\bf step 2:}
The 129{,}188 $(3,3,3;13)$ colorings from step~1 are reduced modulo
weak-isomorphism using \texttt{nauty}\footnote{Note that
  \texttt{nauty} does not handle edge colored graphs and weak
  isomorphism directly. We applied an approach called $k$-layering
  described at
  \url{https://computationalcombinatorics.wordpress.com/2012/09/20/canonical-labelings-with-nauty}.}~\cite{nauty}. This
process results in a set with 78{,}892 graphs.

\begin{table}[t]
\centering\scriptsize
\begin{tabular}{ |c|l|c|c|c|}
\hline
Step&  \multicolumn{1}{|c|}{Notes}&  \multicolumn{3}{|c|} {Computation Times} \\ 
\hline\hline
\hline
\multirow{3}{*}{1}
  & compute all Ramsey $(3,3,3;13)$ 
                           & MiniSAT & CryptoMiniSAT & Glucose \\
                           \cline{3-5}
  & colorings per degree  matrix 
                           & total:~~\hfill 308.23 hr.    
                           & total:~~\hfill 136.31 hr.         
                           & total:~~\hfill 373.2 hr.  \\
  & (999 instances,  129{,}188 solutions)  & hardest:\hfill 9.15 hr.
                           & hardest:\hfill 4.3 hr.
                           & hardest:\hfill 17.67 hr. \\
\hline
{2}
  & reduce modulo $\approx$.  (78{,}892 solutions) 
                           & \multicolumn{3}{|c|}
                                      {\texttt{nauty}, fast (minutes)} \\
\hline
\hline

\end{tabular}
\caption{Computing  $\RR(3,3,3;13)$ step by step.}
\label{tab:333_times}
\end{table}

\paragraph{\bf Does $\mathbf{R(4,3,3)=30}$?}
Note that, in order to prove that there are no $(4,3,3;30)$ colorings
with degrees $(13,8,8)$ using the embedding approach, we would need to
check all embedding instances that contain one of the $(3,3,3;13)$
colorings, a $(4.2.3;8)$ coloring and a $(4,3,2;8)$ coloring. Since
there are $78{,}892$, $3$ and $3$ of these colorings respectively, we
can prove there are no $(4,3,3;30)$ colorings by showing that these
$78,892\times 3\times 3$ embedding instances are unsatisfiable.
We expect that this is the case and in the past three months have
shown that 50\% of the instances are indeed unsatisfiable. Ongoing
computation is proceeding in order to complete the proof.

\section{Conclusion}

We have applied SAT solving techniqes to show that any $(4,3,3;30)$
Ramsey coloring must be $(13,8,8)$ regular in the degrees of the three
colors. In order to apply the same technique to show that there is no
$(13,8,8)$ regular $(4,3,3;30)$ Ramsey coloring we would need to make
use of the set of all $(3,3,3;13)$ colorings. We have computed this
set modulo weak isomorphism. To this end we applied a technique
involving abstraction and symmetry breaking to reduce the redundancies
in the number of isomorphic solutions obtained when applying the SAT
solver. Ongoing computation is proceeding to prove that $R(4,3,3)=30$.

\subsection*{Acknowledgments}
We thank Stanislaw Radziszowski for his guidance and comments which
helped to improve the presentation of this paper. In particular
Stanislaw proposed to show that our techniique is able to find the
$(4,3,3;29)$ coloring depicted as Figure~\ref{embed_12_8_8}.


\end{document}